%% file: main_arxiv.tex
\documentclass[twoside]{article}

\usepackage[accepted]{arxiv}
\usepackage[utf8]{inputenc} 
\usepackage[T1]{fontenc}    
\usepackage{hyperref}       
\usepackage{url}            
\usepackage{booktabs}       
\usepackage{amsfonts}       
\usepackage{nicefrac}       
\usepackage{microtype}      
\usepackage{xcolor}         
\usepackage{makecell}
\usepackage{graphicx}
\usepackage{tabularx}
\usepackage{adjustbox}

\usepackage{amsmath, amsfonts, amsthm,amssymb,here,dsfont, hyperref}
\usepackage{algorithm, algorithmic}
\usepackage{tabularx}

\usepackage{multirow} 
\usepackage{diagbox}
\usepackage{colortbl}
\newcommand{\argmin}[1]{\underset{#1}{\operatorname{arg}\!\operatorname{min}}\;}

\newtheorem{theo}{Theorem}[section]
\newtheorem{definition}[theo]{Definition}
\newtheorem{prop}[theo]{Proposition}

\newtheorem{coro}[theo]{Corollary}
\newtheorem{lem}[theo]{Lemma}

\newtheorem{assu}[theo]{Assumption}

\makeatletter
\newcommand\HUGE{\@setfontsize\Huge{50}{60}}
\makeatother  

\newcommand\norm[1]{\lVert#1\rVert}

\usepackage[textwidth=1.8cm, textsize=tiny]{todonotes}

\renewcommand{\arraystretch}{1}
\usepackage[round]{natbib}
\renewcommand{\bibname}{References}

\begin{document}

%
%

\twocolumn[

\aistatstitle{Parametric Fairness with Statistical Guarantees}


\aistatsauthor{ Fran\c{c}ois HU \And Philipp Ratz \And Arthur Charpentier }

\aistatsaddress{ Université de Montréal\\ francois.hu@umontreal.ca \And Université du Québec à Montréal\\ ratz.philipp@courrier.uqam.ca \And Université du Québec à Montréal\\ charpentier.arthur@uqam.ca} ]

\begin{abstract}

    Algorithmic fairness has gained prominence due to societal and regulatory concerns about biases in Machine Learning models. Common group fairness metrics like Equalized Odds for classification or Demographic Parity for both classification and regression are widely used and a host of computationally advantageous post-processing methods have been developed around them. However, these metrics often limit users from incorporating domain knowledge. Despite meeting traditional fairness criteria, they can obscure issues related to intersectional fairness and even replicate unwanted intra-group biases in the resulting fair solution. To avoid this narrow perspective, we extend the concept of Demographic Parity to incorporate distributional properties in the predictions, allowing expert knowledge to be used in the fair solution. We illustrate the use of this new metric through a practical example of wages, and develop a parametric method that efficiently addresses practical challenges like limited training data and constraints on total spending, offering a robust solution for real-life applications.
\end{abstract}

\section{INTRODUCTION}

To prevent the use of sensitive information such as gender or race in learning algorithms, the field of Algorithmic fairness aims to create predictions that are free of the influences from such variables. Discriminatory biases in real-life datasets lead standard machine learning algorithms to behave unfairly, even when excluding sensitive attributes. This issue has prompted the need to develop methods that optimize prediction performance while satisfying fairness requirements. Several notions of fairness have been considered~\cite{barocas-hardt-narayanan,zafar2019fairness} in the literature. In this paper, we focus on the \emph{Demographic Parity} (DP)~\cite{calders2009building} that requires the independence between the sensitive feature and the predictions, while not relying on labels. The DP-fairness is being pursued extensively in the field, as evidenced by recent research \cite{calders2009building, zemel2013learning, chzhen2019leveraging, agarwal2019fair, elie2021overview, hu2023sequentially}.

Broadly speaking, approaches to obtain algorithmic fairness can be categorized into \textit{pre-processing} methods which enforce fairness in the data before applying machine learning models \cite{calmon2017optimized, adebayo2016iterative}, \textit{in-processing} methods, who achieve fairness in the training step of the learning model \cite{Agarwal_Beygelzimer_Dubik_Langford_Wallach18,Donini_Oneto_Ben-David_Taylor_Pontil18,agarwal2019fair}, and \textit{post-processing} which reduces unfairness in the model inferences following the learning procedure \cite{chiappa2020general,chzhen2020fair, Chzhen_Denis_Hebiri_Oneto_Pontil20Recali,denis2021fairness}. Our work falls into the latter, as this category of algorithms offers computational advantages and are easiest to integrate in existing machine learning pipelines.

Most of the current studies involving post-processing methods employ a neutral approach to enforcing DP-fairness, where model outputs are taken as given and fairness is achieved by constructing a common distribution. However, domain knowledge is often lost when transforming scores without special care. Further, fairness is a multi-faceted issue, where simple optimizations on one metric can lead to new biases in another. Such situations can arise due to issues related to intersectional fairness \cite{foulds2020intersectional}, that is, a population can possess multiple sensitive groups and individuals might reside in an intersection of them. If the marginal predictions for a sensitive group can be further split according to a secondary sensitive variable, simply correcting for one but not the other can have undesirable results. We visualize the issue in the left pane of Figure \ref{fig:pres_unfair}, although an agnostic correction method was chosen, an implicit choice related to the resulting distribution was made. This becomes concerning in the presence of latent sensitive attributes as explicit correction methods such as developed by \cite{hu2023sequentially} cannot be applied directly.
\begin{figure*}[htbp]
\centering
\includegraphics[width=0.85\textwidth]{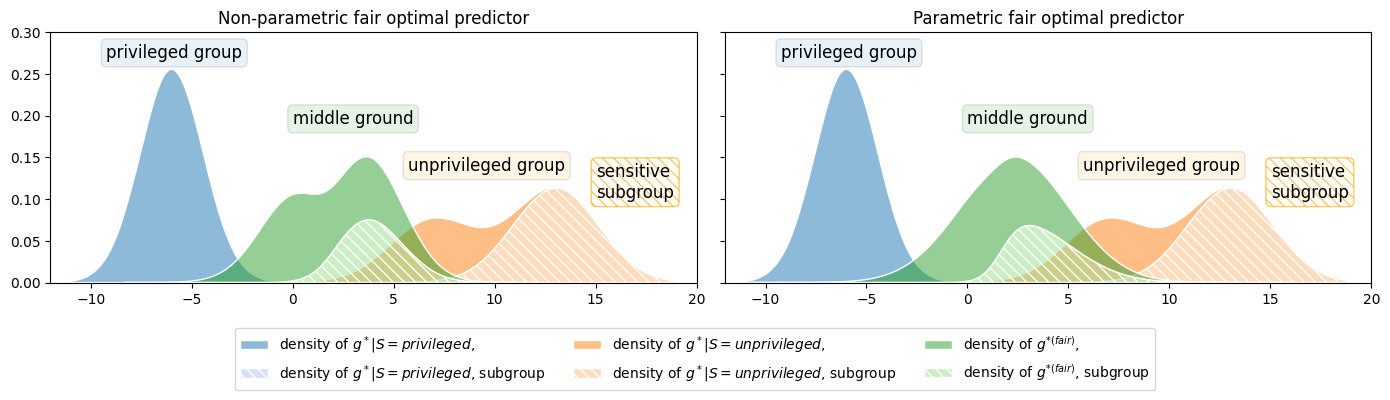}
\caption{Base predictors are shown in blue and orange, while the optimal fair predictor is in green. In this example, integrating fairness considerations with domain knowledge effectively mitigates intersectional fairness issues. Here, $g^*$ corresponds to the Bayes rule, while $g^{*(\text{fair})}$ corresponds to the associated fair optimal predictor.}\label{fig:pres_unfair}
\end{figure*}

From a more practical standpoint, achieving algorithmic fairness also presents challenges that go beyond predictive accuracy under DP-fairness. Two important constraints are overall prediction stability and a smooth transition from unfair to fair regimes. Prediction stability essentially translates to keeping the average score constant, ensuring minimal disturbances to the overall allocations. As a working example, consider a company wishing to achieve fairness in wages with respect to a particular attribute, stability then translates to keeping the overall wage expenses constant before and after fairness enforcing procedures. As there is an inherent trade-off between achieving optimal predictive accuracy and minimal unfairness, a smooth transition means that a fair solution can be achieved across several intermediate steps. Following the example, this would translate to having a transitional period to avoid abrupt changes. Recent literature such as \cite{chzhen2022minimax}, has proposed using relative fairness improvements, which provides a way to achieve a transition to fair results over multiple periods. 

\paragraph{Main Contributions} 
In this article, we propose and study a methodology that tries to satisfy all these points. Summarized, we contribute the following to the field:
\begin{itemize}
    \item We introduce the concept of parametric fair solutions, satisfying shape constraints on fair outcomes. In line with previous research, we develop this method through the use of Wasserstein barycenters. 
    \item We provide an efficient plug-in method and establish fairness and risk guarantees.
    \item Through the use of multiple real-world datasets and different scenarios, we illustrate the effectiveness of our approach. 
\end{itemize}

\paragraph{Related Work}
Within the algorithmic fairness literature, much of the work has developed around Wasserstein barycenters \cite{chiappa2020general, gordaliza2019obtaining, chzhen2020fair} with applications such as investigated in~\cite{ratz2023addressing, charpentier2023mitigating} and our approach can be considered an extension thereof. Of particular use are closed form solutions for the optimal transportation plan, as developed by \cite{chzhen2020fair, gouic2020projection, gaucher2022fair}, which enable a seamless integration of the procedure to most model architectures. Approximate fairness, useful for multi-period transitions, was studies by \cite{chzhen2022minimax}, who proposed a risk-fairness trade-off unfairness measure based on Wasserstein barycenters. 

To obtain parametric solutions based on the Wasserstein distance, the minimum distance approach \cite{basu2011statistical} is of relevance. \cite{bassetti2006asymptotic} proposed an estimation procedure for location-scale models based on the Wasserstein distance, which was then extend to an estimator called the \emph{minimum expected Wasserstein estimator} (MEWE) by \cite{bernton2019parameter}, who also point out the robustness to outliers of the method. 

Whereas both of these fields have independently advanced, there is, to the best of our knowledge, limited exploration into the combination of them. Whereas it seems natural to incorporate domain knowledge into predictions, an inherent difficulty is that optimization approaches often have different metrics. The fact that both the procedures for fairness and minimum distance estimation are based on the Wasserstein distance yields more consistent and interpretable results.

\paragraph{Notation} 
Consider a function $g$ and a random tuple $(\boldsymbol{X}, S)\in\mathcal{X}\times\mathcal{S}\subset\mathbb{R}^d\times \mathbb{N}$, with positive integer $d \geq 1$ and distribution $\mathbb{P}$. Let $\mathcal{V}$ be the space of probability measures on $\mathcal{Y}\subset \mathbb{R}$. Let $\nu_{g}\in\mathcal{V}$ and $\nu_{g|s}\in\mathcal{V}$ be, respectively, the probability measure of $g(\boldsymbol{X}, S)$ and $g(\boldsymbol{X}, S)|S=s$. $F_{g|s}(u) := \mathbb{P}\left( g(\boldsymbol{X}, S) \leq u|S=s \right)$ corresponds to the cumulative distribution function (CDF) of $\nu_{g|s}$ and $Q_{g|s}(v) := \inf\{u\in\mathbb{R}:F_{g|s}(u)\geq v \}$ its associated quantile function. Given a mapping $T: \mathcal{Y} \to  \mathcal{Y}$ with $ \mathcal{Y}\subset \mathbb{R}$, we define the pushforward operator $\sharp$ characterized by $(T\sharp\nu)(\cdot) := \nu \circ T^{-1}(\cdot)$.

\paragraph{Outline of the paper} The article is organized as follows: Section~\ref{sec:background} introduces the Demographic Parity concept of fairness, followed by the presentation of our parametric fairness methodology in Section~\ref{sec:parametric}. We propose, in Section~\ref{sec:datadriven}, a data-driven approach where we establish fairness and estimation guarantees. The performance of our estimator is assessed on real data in Section~\ref{sec:Evaluation}, and we draw conclusions in Section~\ref{sec:conclusion}.

\section{BACKGROUND ON FAIRNESS UNDER DEMOGRAPHIC PARITY}
\label{sec:background}
Let $(\boldsymbol{X},S,Y)$ be a random tuple with distribution $\mathbb{P}$. $\boldsymbol{X}\in\mathcal{X}\subset \mathbb{R}^d$ represents the non-sensitive features, $Y\in \mathcal{Y} \subset \mathbb{R}$ represents the task to be estimated, and $S\in \mathcal{S}:= \{1, \dots, K\}$ a discrete sensitive feature with distribution $(p_s)_{s\in\mathcal{S}}$ where $p_s := \mathbb{P}(S=s)$ and we assume $\min_{s\in\mathcal{S}} \{p_s\} > 0$. We denote $\mathcal{G}$ the class of all predictors of the form $g:\mathcal{X}\times\mathcal{S}\to \mathcal{Y}$ that have an absolutely continuous \textit{w.r.t.} the Lebesgue measure. More precisely, we require the following assumption:

\begin{assu}\label{assu:general}
    For $g\in\mathcal{G}$, measures $\{\nu_{g|s}\}_{s\in\mathcal{S}}$ are non-atomic with finite second moments.
\end{assu}

\paragraph*{Risk measure} 

We focus on the regression case, although our findings are extendable to the classification case, see~\cite{gaucher2022fair}. Our objective is to minimize the squared risk in $\mathcal{G}$. Notably, recall that the Bayes regressor $g^*(\boldsymbol{X}, S) := \mathbb{E}[Y | \boldsymbol{X}, S]$ corresponds to the optimal predictor that minimizes squared risk,
\begin{equation}\label{eq:OptRisk}
\text{(\textbf{Risk measure})}\quad\mathcal{R}(g) := \mathbb{E}(Y - g(\boldsymbol{X}, S) )^2\enspace.
\end{equation}
The optimal risk, is defined as $\mathcal{R}^* := \inf_{g\in\mathcal{G}}\mathcal{R}(g)$ and for any subclass $\mathcal{G}'\subset \mathcal{G}$, the \textit{excess-risk} of the class $\mathcal{G}$ is defined by
\begin{equation*}\label{eq:ExcessRisk}
\mathcal{E}(\mathcal{G}') := \inf_{g\in\mathcal{G}'}\mathcal{R}(g) - \mathcal{R}^*\enspace.
\end{equation*}
This helps to quantify performance disparities among predictors that impose conditions on the class $\mathcal{G}$, such as ensuring fairness (denoted $\mathcal{G}^0$) or limiting predictors to specific distributions (denoted $\mathcal{G}_{\Theta}$), or both ($\mathcal{G}_{\Theta}^0$).

\paragraph*{Demographic Parity}
For a predictor $g\in\mathcal{G}$, the (Strong) Demographic Parity (DP) is satisfied if the probability assigned is invariant across the values of the sensitive attributes, i.e., for all $s\in \mathcal{S}$,
\begin{multline*}
    \label{eq:DPFair}
\sup\limits_{u\in\mathbb{R}} \left| \mathbb{P}(g(\boldsymbol{X}, S) \leq u) - \mathbb{P}(g(\boldsymbol{X}, S) \leq u | S = s) \right| = 0\enspace,
\end{multline*}
or equivalently with quantiles,
\begin{equation*}
    \max_{s\in \mathcal{S}}\int_0^1 \left|\ Q_{g}(u) - Q_{g|s}(u)\ \right|du = 0 \quad \enspace.
\end{equation*}
To extend this last definition to probability measures, we classically consider the Wasserstein distance, defined below. 
\begin{definition}[Wasserstein distances]\label{def:W2}
    Let $\nu$ and $\nu'$ be two probability measures. The $p$-Wasserstein distance between $\nu$ and $\nu'$ is defined as
    \begin{equation*}
        \mathcal{W}_p(\nu, \nu') = \left(\inf_{\pi\in\Pi_{\nu, \nu'}} \left\{ \int_{\mathcal{Y}\times\mathcal{Y}} (Y-Y')^p d\pi(Y, Y') \right\}\right)^{1/p}\enspace,
    \end{equation*}
    where $\Pi_{\nu, \nu'}$ is the set of distributions on $\mathcal{Y}\times\mathcal{Y}$ having $\nu$ and $\nu'$ as marginals. The coupling $\pi$ which achieves the infimum is called the optimal coupling.
\end{definition}

Further, if one measure in the p-Wasserstein distance has a density, the optimal coupling is deterministic (\cite{santambrogio2015optimal} Thm.~2.9). Given $X\sim \nu'$ and assuming $\nu$ has a density, a mapping $T:\mathbb{R}\to \mathbb{R}$ exists (and is unique if $p>1$), satisfying $T\sharp\nu = \nu'$ and
$$
\mathcal{W}_p^p(\nu, \nu') = \mathbb{E}(X-T(X))^p\enspace.
$$
The $p$-Wasserstein distance between two univariate measures $\nu$ and $\nu'$ can also be expressed by quantiles 
$$
\mathcal{W}_p^p\left( \nu, \nu' \right) = \int_0^1 \left|\ Q_{\nu}(u) - Q_{\nu'}(u)\ \right|^p du\enspace,
$$
which expresses the link between DP fairness and the widespread use of the Wasserstein distance within the field. Indeed, the unfairness of a predictor $g\in\mathcal{G}$ can be quantified by the unfairness measure,
\begin{multline}
    \label{eq:Unfairness}
    \text{(\textbf{Unfairness})}\quad\mathcal{U}(g) := \max_{s\in\mathcal{S}}\mathcal{W}_1\left( \nu_{g}, \nu_{g|s} \right)\enspace.
\end{multline}
We express below the \textit{exact} and \textit{approximate} DP fairness through the 1-Wasserstein distances, with the notion of Relative Improvement (RI) first introduced in \cite{chzhen2022minimax}.
\begin{definition}[Fairness under Demographic Parity]\label{def:Unfairness}
Given an RI $\varepsilon\geq 0$, a predictor $g$ is called approximately fair under DP if and only if $\mathcal{U}(g) \leq  \varepsilon\times \mathcal{U}(g^*)$. In particular, $g$ is called exactly fair if and only if $\mathcal{U}(g) = 0$.
\end{definition}

Recall that $\mathcal{G}$ represents the class of all predictors verifying A.~\ref{assu:general}. We denote $\mathcal{G}^0$ the class of exactly DP-fair predictors, i.e.,
$$
\mathcal{G}^0 := \left\{ g\in\mathcal{G}: \mathcal{U}(g) = 0\right\}\enspace,
$$
In the context of approximate fairness, our focus lies in the relative improvement of a fair predictor compared to Bayes' rule $g^*$. Considering this framework, we extend the $\mathcal{G}^0$ notation, denoting for any $\varepsilon\geq 0$,
$$
\mathcal{G}^{\varepsilon} := \{g\in\mathcal{G}: \mathcal{U}(g) \leq \varepsilon\times\mathcal{U}(g^*)\}\enspace.
$$
the set of all $\varepsilon$-RI fair predictors in $\mathcal{G}$. In particular, for all $\varepsilon \leq \varepsilon'$ in $[0, 1]$, $\mathcal{G}^{0}\subset\mathcal{G}^{\varepsilon}\subset \mathcal{G}^{\varepsilon'}\subset\mathcal{G}^{1}$ where $g^*\in\mathcal{G}^1$ and $\mathcal{G}^0$ corresponds to the set of exactly DP-fair predictors. 

Let us now turn our attention to exact fairness, which will later be extended to the approximate methodology in Section~\ref{subsec:approx_fairness}.

\paragraph*{Optimal Fair Predictor For Exact Fairness}
The problem of optimal prediction has been well studied. For example, \cite{chzhen2020fair, gouic2020projection, gaucher2022fair, hu2023fairness} use the optimal transport theory to develop fair solutions for various tasks, such as classification or regression, or both. Indeed,
given $\varepsilon=0$, the excess-risk is given by,
\begin{equation}\label{eq:FairExcessRisk}
\mathcal{E}(\mathcal{G}^0) = \min_{g\in\mathcal{G}}\sum_{s\in\mathcal{S}} p_s\mathcal{W}_2^2(\nu_{g^*|s}, \nu_{g})\enspace.
\end{equation}
Additionally, this expression gives us a fair optimal predictor denoted $g^{(0)*}$ of the form,
\begin{equation}\label{eq:FairPredTh}
    g^{(0)*}(\boldsymbol{x}, s) = T_{g^*|s \,\to\, \text{bary}} \left( g^*(\boldsymbol{x}, s)\right) ,\quad (\boldsymbol{x}, s)\in\mathcal{X}\times\mathcal{S} \enspace,
\end{equation}
where $T_{g^*|s \,\to\, \text{bary}}$ 
is the optimal transport map from $\nu_{g^*|s}$ to the Wasserstein barycenter. A closed-form solution can be explicitly derived as follows:
    $$T_{g^*|s \,\to\, \text{bary}}(\cdot) = \left(\sum_{s'\in\mathcal{S}}p_{s'}Q_{g^*|s'}\right)\circ F_{g^*|s}(\cdot)\enspace.$$
This outcome enables precise fair learning through post-processing, as illustrated in the left pane of Fig.~\ref{fig:pres_unfair} (see green density). In the next section, we extend this result to \textit{parametric fairness}, allowing the incorporation of expert knowledge into the fair solution and showcasing favorable distributional properties.

\section{PARAMETRIC DEMOGRAPHIC PARITY}\label{sec:parametric}
We examine the impact of imposing a specific shape constraint on the optimal fair predictor. This constraint narrows down our focus to a subset of parametrized predictors, denoted as $\mathcal{G}_\Theta$. Our estimations are confined within this subset for analysis.

\subsection{Distributional Constraints} 

It is worth noting that a distributional constraint imposes limitations on the estimation, but can actually help achieve more specific goals. Hence, we use the term constraint in the optimization sense here. Before listing the technical details, we present a short motivation for the use of a parametric subclass $\mathcal{G}_\Theta \subset \mathcal{G}$, referring to this restriction as the class of parametric predictors. In this paper, we refer to $\mathcal{G}_\Theta$ as a family of continuous distributions. 


\subsubsection{Domain Expertise}

The choice of the family of distribution $\mathcal{G}_\Theta$ is contingent upon both the specific application and its associated social considerations. For example, some score are supposed to follow a specific distribution:

\paragraph*{Gaussian Distribution $\mathcal{G}_\Theta = \mathcal{N}(\mu, \sigma^2)$} For instance, this distributional constraint works in scenarios such as university grading systems, where grades are expected to follow a Gaussian pattern (centered around $\mu$ with variance $\sigma^2$) devoid of racial bias. 



\subsubsection{Indirectly Mitigating Intersectional Unfairness and Practical Considerations}

Moving beyond traditional fairness evaluations based on entire groups (such as Demographic Parity), "distributional unfairness" acknowledges biases within specific sections of unprivileged groups. While some areas might seem just, others suffer from injustice. Fairness is not only about overall group comparison; it involves recognizing unfairness in specific treatment aspects. For instance, only focusing on a single sensitive attribute is insufficient~\cite{kong2022intersectionally}; it overlooks intersecting subgroups, leading to \textit{fairness gerrymandering}~\cite{kearns2018preventing}. This term describes the problem when unfairness is assessed only over a few arbitrarily chosen groups. As an example, it was revealed that algorithms recognized women with darker skin tones with reduced accuracy, leading to different treatments (i.e, output distributions) within women population. Finding a simple middle ground using the Wasserstein barycenter can hence lead to disadvantages for subgroups within the population. 



\paragraph*{Representation Bias} In machine learning, representation bias occurs when models exhibit lower performance for demographic groups that are underrepresented in the training data. This discrepancy can lead to significant disparities in outcomes. One way to address this bias is through the parametric fairness approach, which establishes a shared distribution, or \textit{belief}, among both privileged and unprivileged groups. By doing so, this approach helps mitigate representation bias, enhancing the model's fairness and accuracy across diverse demographics.

\paragraph*{Mean Output Preservation} 
To lay the groundwork for studying parametric fairness, we first need to study changes between our optimal fair mean predictions and uncalibrated mean prediction $\mathbb{E}[g^*(\boldsymbol{X}, S)]$. For instance, in predicting an individual's wage using $g\in\mathcal{G}$, the \textit{budget deviation} refers to the deviation from the initial mean output as measured by 
$$
\mathcal{D}(g) := \mathbb{E}[g(\boldsymbol{X}, S) - g^*(\boldsymbol{X}, S)]\enspace.
$$
If $\mathcal{D}(g)=0$, we achieve \textit{mean output preservation}. Here, $\nu_{g^{(0)*}}$ represents the Wasserstein barycenter with weights $(p_s)_{s\in\mathcal{S}}$ and means $(m^*_s)_{s\in\mathcal{S}}$ where $m^*_s:=\mathbb{E}_{\boldsymbol{X}|S=s}[g^*(\boldsymbol{X}, S)]$. This further ensures $\mathcal{D}(g^{(0)*})=0$ due to the barycenter's mean property,
$$
\mathbb{E}[g^{(0)*}(\boldsymbol{X}, S)] = \sum_{s\in\mathcal{S}}p_s\cdot m^*_s = \mathbb{E}[g^*(\boldsymbol{X}, S)]\enspace.
$$
Specifically, we are interested in evaluating the amount of information lost (risk, unfairness and budget) when constraining to the subclass $\mathcal{G}_\Theta$. We denote $\mathcal{G}_\Theta^0$ the class of DP-fair predictor in $\mathcal{G}_\Theta$ and define and quantify the information loss as bellow.



\subsection{Parametric Exactly Fair Predictor}

The bound on the amount of information lost due to the class constraint at $\mathcal{G}_{\Theta}$ can be described as:
\begin{prop}[Exact parametric fair predictor] \label{prop:ParamFairPredictor}
Assume that A.~\ref{assu:general} hold, then, the excess-risk can be quantified by
    \begin{equation}\label{eq:ExcessRiskParam}
        \mathcal{E}(\mathcal{G}_\Theta^0) = \inf_{g_\theta \in \mathcal{G}_\Theta}\sum_{s\in\mathcal{S}}p_s \mathcal{W}_2^2\left(\nu_{g^*|s}, \nu_{g_\theta}\right)\enspace.
    \end{equation}
    In addition, if we denote $g_\theta^{(0)*}$ the minimizer of the r.h.s. of Eq.~\eqref{eq:ExcessRiskParam}, we can bound the excess-risk of $\mathcal{G}_\Theta^0$ and the budget deviation of $g_\theta^{(0)*}$ as follows:
    \begin{multline*}
        \mathcal{E}(\mathcal{G}^0)
        \leq \mathcal{E}(\mathcal{G}_\Theta^0)
        \leq 2\left(\mathcal{E}(\mathcal{G}^0) + \inf_{g_\theta \in \mathcal{G}_\Theta} \mathcal{W}_2^2(\nu_{g^{(0)*}}, \nu_{g_\theta})\right)
    \end{multline*}
and,
$$ 0\leq\mathcal{D}(g_\theta^{(0)*})^2 \leq \mathcal{W}_2^2\left(\nu_{g^{(0)*}}, \nu_{g_\theta^{(0)*}}\right) \enspace.$$


\end{prop}

Prop.~\ref{prop:ParamFairPredictor} indicates that within a subclass $\mathcal{G}_{\Theta}$, information loss is partially controlled by the minimum 2-Wasserstein distance to the true Wasserstein barycenter $g^{(0)*}$. However, a direct computation of the constrained Wasserstein barycenter $\mathcal{E}(\mathcal{G}_\Theta^0)$ is prohibitively complex and we instead propose an adequate approximation within the subclass $\mathcal{G}_{\Theta}$. Note that, for any $g_\theta \in \mathcal{G}_\Theta$, we obtain:
\begin{multline*}
        \sum_{s\in\mathcal{S}}p_s\mathcal{W}^2_2\left( \nu_{g_\theta}, \nu_{g^*|s} \right) \leq 2 \left(\mathcal{W}^2_2\left( \nu_{g_\theta}, \nu_{g^{(0)*}} \right) +  \mathcal{E}(\mathcal{G}^0)\right)\enspace.
\end{multline*}
If we assume the set $\boldsymbol{X} \times S$ is compact (therefore bounded), especially if the diameter verifies $\textrm{diam}(\boldsymbol{X} \times S) \leq 1$, we have:
\begin{multline*}
        \sum_{s\in\mathcal{S}}p_s\mathcal{W}^2_2\left( \nu_{g_\theta}, \nu_{g^*|s} \right) \leq 2 \left(\mathcal{W}_1\left( \nu_{g_\theta}, \nu_{g^{(0)*}} \right) + \mathcal{E}(\mathcal{G}^0)\right)\enspace,
\end{multline*}
which holds true when a simple normalization step is applied, scaling every feature within the range of $[0, 1]$. These upper bounds suggest that the best parametric fair predictor $g_\theta^{(0)*}$, considering risk, unfairness and budget, can be approximated within the subclass $\mathcal{G}_{\Theta}$ by minimizing the 2-Wasserstein (or 1-Wasserstein) distance to the actual Wasserstein barycenter $g^{(0)*}$.


\subsection{Extension To Approximate Fairness}
\label{subsec:approx_fairness}
The approximate framework aims to achieve \textit{approximate} fairness by finding an optimal $\varepsilon$-RI fair predictor, minimizing $\inf_{g\in\mathcal{G}^\varepsilon} \mathcal{R}(g)$ for $\varepsilon\in [0, 1]$. Extending Prop.~\ref{prop:ParamFairPredictor} to this end, we show that a solution using the geodesic approach can map any exact fair predictor (including $g^{(0)*}$) to an approximate one in $\mathcal{G}^\varepsilon$. This is achieved by introducing the geodesic paths in 2-Wasserstein.



\paragraph{Geodesic Interpolation} A curve of probability measures $(\nu_{\varepsilon})_{\varepsilon\in[0, 1]}$ is called a (constant-speed) geodesic in the 2-Wasserstein space (\cite{ambrosio2005gradient} \S2.4.3) if
$$
\mathcal{W}_2(\nu_{\varepsilon}, \nu_0) = \varepsilon \cdot \mathcal{W}_2(\nu_{1}, \nu_0),\quad \varepsilon\in[0, 1]\enspace.
$$
In particular, if we denote $T_{\nu_0\to \nu_1}$ the 
optimal mapping from $\nu_0$ to $\nu_1$ then the corresponding 
geodesic curve is
$$
\nu_{\varepsilon} = \left( (1-\varepsilon)\cdot Id + \varepsilon\cdot T_{\nu_0\to \nu_1} \right)\sharp \nu_0,\quad \varepsilon\in[0, 1]\enspace.
$$
Note that this geodesic curve is unique in the 2-Wasserstein space (\cite{kloeckner2010geometric}, \S 2.2).

We use the geodesic curve to approximate appropriately a fair predictor based on an exact fair one. More specifically, we consider the geodesic paths~\cite{villani2003topics, santambrogio2015optimal} $(g^{(\varepsilon)})_{\varepsilon\in[0, 1]}$ in 2-Wasserstein space between \textbf{any} DP-constrained predictor $g^{(0)}\in \mathcal{G}^{(0)}$ and the unconstrained optimal predictor $g^*$,
\begin{equation}\label{eq:EpsFairPredTh}
g^{(\varepsilon)}(\boldsymbol{X}, S) = (1-\varepsilon)\cdot g^{(0)}(\boldsymbol{X}, S) + \varepsilon \cdot g^{*}(\boldsymbol{X}, S)\enspace.
\end{equation}
This approach in Algorithmic Fairness is also known as \textit{Geometric Repair} \cite{feldman2015certifying, gordaliza2019obtaining}. See Fig.~\ref{fig:univariate_gaussian_epsfair} for an illustration of geodesic paths. This expression allows us to derive directly the following Lemma:
\begin{lem}[Risk-unfairness trade-off]\label{lem:RiskUnf} Given $\varepsilon\in[0, 1]$ and any predictor $g^{(0)}\in \mathcal{G}^{(0)}$, $g^{(\varepsilon)}$ satisfies,
    $$
    \mathcal{R}(g^{(\varepsilon)}) = (1-\varepsilon)^2\times \mathcal{R}(g^{(0)})
    \quad\text{and}\quad
    \mathcal{U}(g^{(\varepsilon)}) = \varepsilon\times \mathcal{U}(g^{*})
    $$
\end{lem}
If we replace $g^{(0)}$ with $g^{(0)*}$ in Eq.~\eqref{eq:EpsFairPredTh}, then the results in Lemma~\ref{lem:RiskUnf} hold and we denote the result as $g^{(\varepsilon)*}$, where $\varepsilon$ controls the distance to the Wasserstein barycenter $g^{(0)*}$. Notably, for any $g^{(0)}\in\mathcal{G}^0$, $\mathcal{R}(g^{(\varepsilon)*})\leq \mathcal{R}(g^{(\varepsilon)})$ while having the same level of unfairness. 
Moreover, as per \cite{chzhen2022minimax} (Prop. 4.1), $g^{(\varepsilon)*}$ represents the optimal fair predictor with $\varepsilon$-RI, minimizing the risk
$\inf_{g\in\mathcal{G}^\varepsilon} \mathcal{R}(g)$.
\begin{figure}[h!]
\centering
\includegraphics[width=0.45\textwidth]{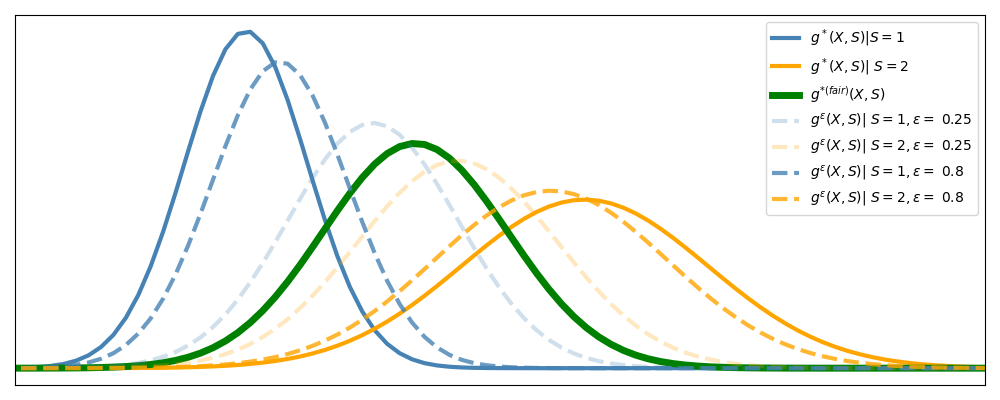}
\caption{Approximate fairness between two gaussian distributions and their barycenter} \label{fig:univariate_gaussian_epsfair}
\end{figure}

For any $\varepsilon\in[0, 1]$, $g^{(\varepsilon)*}$ exhibits budget stability: $\mathcal{D}(g^{(\varepsilon)*}) = \mathcal{D}(g^{(0)*}) = 0$, which implies that the $(g^{(\varepsilon)*})_{\varepsilon}$ curve adheres to the initial allocation budget.

\paragraph*{Parametric Case} In line with the previous approach, we consider $(h^{(\varepsilon)})_{\varepsilon\in[0, 1]}$ the geodesics between any parametric DP-fair predictor $g^{(0)}_\theta\in\mathcal{G}_\Theta^0$ and $g^*$ defined as
\begin{equation}\label{eq:EpsParamFairPredTh}
h^{(\varepsilon)}(\boldsymbol{X}, S) := (1-\varepsilon)\cdot g_\theta^{(0)}(\boldsymbol{X}, S) + \varepsilon \cdot g^{*}(\boldsymbol{X}, S)\enspace,
\end{equation}
which corresponds to a $\varepsilon$-RI predictor within a subclass of $\mathcal{G}^\varepsilon$, denoted 
$$
\mathcal{H}^\varepsilon := \left\{ h^{(\varepsilon)} \in \mathcal{G}^\varepsilon :  g_\theta^{(0)}\in\mathcal{G}_\Theta^0 \ s.t. \ h^{(\varepsilon)} \text{ verifies Eq.~\eqref{eq:EpsParamFairPredTh}} \right\}\enspace,
$$
with $\mathcal{H}^0 = \mathcal{G}_\Theta^0$ as a parametric subclass and $\mathcal{H}^1 = \mathcal{G}$ as a non-parametric subclass. The following proposition establishes an upper bound on the information loss caused by imposing Eq.~\eqref{eq:EpsParamFairPredTh}.

\begin{prop}[Approximate parametric fairness] \label{prop:ParamFairPredictorApprox}
Assume that A.~\ref{assu:general} hold, then,
\begin{enumerate}
    \item [i)] (Risk-unfairness trade-off) Given $\varepsilon\in[0, 1]$ and any $g_\theta^{(0)} \in\mathcal{G}^0_\Theta$, we have
    $$
    \mathcal{R}(h^{(\varepsilon)}) = (1-\varepsilon)^2\times \mathcal{R}(g_\theta^{(0)})\enspace,
    $$
 and
 $$
    \mathcal{U}(h^{(\varepsilon)}) = \varepsilon\times \mathcal{U}(g^{*})\enspace.
    $$
Replacing $g_\theta^{(0)}$ with $g_\theta^{(0)*}$ in Eq.~\eqref{eq:EpsParamFairPredTh} yields the optimal predictor, denoted $h^{(\varepsilon)*}$, in $\mathcal{H}^\varepsilon$. Notably, $h^{(\varepsilon)*}$ is the best risk-optimal choice among all interpolated models of Eq.~\eqref{eq:EpsParamFairPredTh} with equal unfairness levels.
    \item [ii)] (Upper-bounded excess-risk)
Additionally, the resulting excess-risk can be bounded by:
\begin{multline*}
    \mathcal{E}(\mathcal{G}^\varepsilon) \leq \mathcal{E}(\mathcal{H}^\varepsilon) \leq \\
     2(1-\varepsilon)^2\big(\mathcal{E}(\mathcal{G}^0) + \inf_{g_\theta \in \mathcal{G}_\Theta} \mathcal{W}_2^2(\nu_{g^{(0)*}}, \nu_{g_\theta}) \big)\enspace.
\end{multline*}
\item[iii)] (Bound on budget deviation) The squared budget deviation of $h^{(\varepsilon)*}$ is bounded by:
\begin{multline*}
    \mathcal{D}(h^{(\varepsilon)*})^2 \leq (1-\varepsilon)^2\cdot \mathcal{D}(g_\theta^{(0)*})^2\\
    \leq (1-\varepsilon)^2\cdot\mathcal{W}_2^2\left(\nu_{g^{(0)*}}, \nu_{g_\theta^{(0)*}}\right) \enspace.
\end{multline*}

\end{enumerate}

\end{prop}
Similarly to Prop.~\ref{prop:ParamFairPredictor}, the bound in Prop.~\ref{prop:ParamFairPredictorApprox}-\textit{ii)} suggests that the information loss is partially controlled by $\inf_{g_\theta \in \mathcal{G}_\Theta} \mathcal{W}_2^2(\nu_{g^{(0)*}}, \nu_{g_\theta})$, and its minimizer, denoted $\Tilde{g_\theta}$, can serve as a good approximation of $g^{(0)*}_\theta$. Further Prop.~\ref{prop:ParamFairPredictorApprox}-\textit{iii)} shows that budget stability is maintained with a suitably chosen distribution family $\mathcal{G}_\Theta$, close in "shape" distribution to the barycenter $g^{(0)*}$.

To improve our approach beyond the naive method, we then propose a simple three-step estimation procedure. We sequentially construct the predictors $\left(g^{(0)*}, \Tilde{g_\theta}, \Tilde{h}^{(\varepsilon)}\right)$. Firstly, we construct the optimal fair regressor $g^{(0)*}$ via the Wasserstein barycenter. Next, we compute the parametric fair regressor $\Tilde{g_\theta}$ as the minimizer of the Wasserstein distance to the true barycenter $g^{(0)*}$.
Finally, through geodesic interpolation using $\Tilde{g_\theta}$, we determine the approximately fair predictor 
$$
\Tilde{h}^{(\varepsilon)}(\boldsymbol{X}, S) = (1-\varepsilon)\cdot \Tilde{g_\theta}(\boldsymbol{X}, S) + \varepsilon \cdot g^{*}(\boldsymbol{X}, S)\enspace.
$$
Note that, although partially parameterized by $\theta$, both $h^{(\varepsilon)*}$ and $\Tilde{h}^{(\varepsilon)}$ are not necessarily members of the parametric class $\mathcal{G}_\Theta$, contrary to $h^{(0)*} = g^{(0)*}_\theta$ and $\Tilde{h}^{(0)} = \Tilde{g_\theta}$.


\section{DATA-DRIVEN PROCEDURE}
\label{sec:datadriven}

This section proposes a plug-in estimator for our methodology using empirical data. The construction details are in Section~\ref{subsec:plugin}, and its statistical properties are discussed in Section~\ref{subsec:guarantee}.

\subsection{Plug-in Estimator}
\label{subsec:plugin}

In line with previous research, we start from the unconstrained optimal estimator $\hat{g}$ of $g^*$ trained on a training data set, and an unlabeled calibration set $\mathcal{D}_N^{\text{calib}} = (\boldsymbol{X}_i, S_i)_{i=1}^{N}$ i.i.d. copies of $(\boldsymbol{X}, S)$. Both $\hat{g}$ and $\mathcal{D}_N^{\text{calib}}$ are then used to compute the empirical counterpart $\hat{g}^{(0)}$ of the optimal fair predictor $g^{(0)*}$ defined in Eq.~\eqref{eq:FairPredTh}, following the methodology outlined in \cite{Chzhen_Denis_Hebiri_Oneto_Pontil20Wasser, gaucher2023fair}. In addition, a set of parameters $\theta$ to be estimated is required, which are estimated using the minimum expected distance.

\paragraph*{Minimum Expected Wasserstein Estimation}

To find the parameter associated with the Wasserstein barycenter distribution, we can use the results from \cite{bernton2019parameter}. They show that under mild conditions, the MEWE exists and is consistent. That is, for the true distribution, here denoted $\nu_{\ast}$, the empirical distribution $\nu_n$ and the model distribution $\nu_{g_\theta}$, the minimum of $\theta \mapsto \mathcal{W}_p(\nu_n, \nu_{g_\theta})$ converges to the minimum of $\theta \mapsto \mathcal{W}_p(\nu_{\ast}, \nu_{g_\theta})$.

Crucially, they show under model misspecifiction that the MEWE does not necessarily converge to the same parameter as the maximum likelihood approach. Given that we want the perturbations introduced by the parametric form to be minimal with respect to the transport metric, which seems like a desirable property.


\begin{algorithm}
   \caption{Parametric fair predictor}
   \label{alg:optimization}
\begin{algorithmic}
   \STATE {\bfseries Input:} base estimator $\hat{g}$, unlabeled sample $\mathcal{D}^{\text{calib}}_N$, new data point $(\boldsymbol{x}, s)$, parameter to be estimated $\theta$.

   \STATE {\bf \quad Step 0.} Based on $\mathcal{D}^{\text{calib}}_N$ and $\Hat{g}$, compute the empirical counterpart of $\{p_s\}_s$, $F_{g|s}$ and $Q_{g|s}$.

   \STATE {\bf \quad Step 1.} Then compute the empirical version $\Hat{g}^{(0)}$ of Eq.~\eqref{eq:EpsFairPredTh};

   \STATE {\bf \quad Step 2.} Estimate $\hat{\theta}$ using the appropriate distance metric and parametric form. Sample from resulting distribution and create mapping function between $\Hat{g}^{(0)}$ and $\Hat{g}_\theta$;

   \STATE {\bf \quad Step 3.} Use geodesic interpolation to get $\hat{h}^{(\varepsilon)}$:
   $$
   \hat{h}^{(\varepsilon)}(\boldsymbol{x},s) = (1-\varepsilon)\cdot \Hat{g}_\theta(\boldsymbol{x},s) + \varepsilon \cdot \Hat{g}(\boldsymbol{x},s)\enspace;
   $$
   
   \STATE {\bfseries Output:} parametric approximately fair predictors $\hat{h}^{(\varepsilon)}(\boldsymbol{x},s)$ at point $(\boldsymbol{x},s)$.

\end{algorithmic}
\end{algorithm}



Finally, we can compute the optimal approximately fair predictor $\hat{h}^{(\varepsilon)}$ through the geodesic interpolation between the parametric fair estimator and the optimal unconstrained estimator. 

\subsection{Statistical Guarantees}
\label{subsec:guarantee}

We establish the estimation guarantee before delving into the fairness guarantee. Note that we have adapted the estimation guarantee sequentially from \cite{gouic2020projection}, and \cite{bernton2019parameter} to account for the parametric framework. 

We denote by $\Hat{\nu}_g$ the classical empirical measure of $\nu_g$ of the form $\Hat{\nu}_g := \frac{1}{n}\sum_{i=1}^{n}\delta_{g(\boldsymbol{x}_i, s_i)}$ where $\delta_{g(\boldsymbol{x}_i, s_i)}$ is the Dirac distribution with mass on $g(\boldsymbol{x}_i, s_i)\in\mathcal{Y}$. In addition to A~\ref{assu:general}, we also require the following technical condition.
\begin{assu}[Smoothness \& Bound assumption]\label{assu:LipsGen}
    We assume that $(\Hat{g}(\cdot, s) )_{s\in\mathcal{S}}$ are uniformly Lipschitz and the estimator $\Hat{g}$ is bounded.
\end{assu}
With this assumption, we are then able to derive the following estimation guarantee for $\Hat{g}^{(0)}$.
\begin{lem}[adapted from \cite{gouic2020projection} Thm.~8]
\label{lem:gouic}
    Under A.~\ref{assu:LipsGen}, and assuming the $L_2$-consistency
$$
\mathbb{E}[(g^*(\boldsymbol{X}, S) - \hat{g}(\boldsymbol{X}, S))^2] \to 0 \ \text{as }n\to \infty\enspace,
$$
we have both,
$$
\mathcal{W}_2^2 (\Hat{\nu}_{\hat{g}|s}, \nu_{g^{*}|s}) \to 0 \quad \text{and}\quad \mathcal{W}_2^2 (\Hat{\nu}_{\hat{g}^{(0)}|s}, \nu_{g^{(0)*}}) \to 0\quad\text{a.s.}
$$  
\end{lem}
We establish further statistical guarantees for $\Hat{h}^{(\varepsilon)}$. In addition to the aforementioned assumptions and under mild assumptions specified in Appx.~\ref{subappx:cvg} \cite{bernton2019parameter} (Th.2.4) shows that using MEWE as $n\to +\infty$,
\begin{equation}
\label{eq:mewe}
    \inf_{\Hat{g}_\theta\in\mathcal{G}_\Theta} \mathcal{W}_2(\Hat{\nu}_{\Hat{g}^{(0)}}, \Hat{\nu}_{\Hat{g}_{\theta}}) \to \inf_{g_\theta\in\mathcal{G}_\Theta}\mathcal{W}_2(\nu_{g^{(0)*}}, \nu_{g_\theta})\quad a.s.
\end{equation}
From the results above and given any $\varepsilon \in [0, 1]$ we can directly state the following corollary:

\begin{coro}[Consistency for $\varepsilon \geq 0$]
\label{coro:estimation}
Let $\varepsilon\in[0, 1]$, if $\mathbb{E}\norm{\Hat{g}_\theta(\boldsymbol{X}, S) - \Tilde{g_\theta}(\boldsymbol{X}, S)}^2 \underset{n\to\infty}{\longrightarrow} 0$, then,
$$
\Hat{\nu}_{\Hat{h}^{(\varepsilon)}|s}\underset{n\to\infty}{\longrightarrow} \nu_{h^{(\varepsilon)*}|s}\quad \text{in } \mathcal{W}_2 \text{ a.s.}\quad\enspace.
$$
\end{coro}

It is straightforward to extend the results from Eq.~\eqref{eq:mewe} to include the fairness as well.

\paragraph*{Fairness Guarantee}
Given Eq.~\eqref{eq:mewe}, we can provide a fairness guarantee for the $\varepsilon$ case

\begin{coro}[$\varepsilon$-RI fairness guarantee]
\label{coro:fairness}
    For all $\varepsilon\geq 0$,
    $$
\mathcal{U}\left(\Hat{h}^{(\varepsilon)}\right) = 
    \max_{s\in\mathcal{S}}
\mathcal{W}_1\left(\Hat{\nu}_{\Hat{h}^{(\varepsilon)}|s}, \Hat{\nu}_{\Hat{h}^{(\varepsilon)}}\right)\leq \varepsilon\cdot \mathcal{U}(g^*) + C_n'  \quad\enspace,
    $$
    where $C_n'\to 0$ in $\mathcal{W}_1$ a.s. when $n\to+\infty$.
\end{coro}

Therefore, $\Hat{h}^{(\varepsilon)}$ is asymptotically approximately fair with $\varepsilon$-RI. Although we assume $\Hat{g}$ to be $L_2$-consistent, it is worth noting that this corollary still holds even if $\Hat{g}$ is $L_1$-consistent. Thus the provided methodology offers, under some conditions, a post-processing methodology where fairness and risk guarantees are well established.

\section{NUMERICAL EXPERIMENTS}\label{sec:Evaluation}

\begin{figure*}[htbp]
\centering
\includegraphics[width=0.95\textwidth]{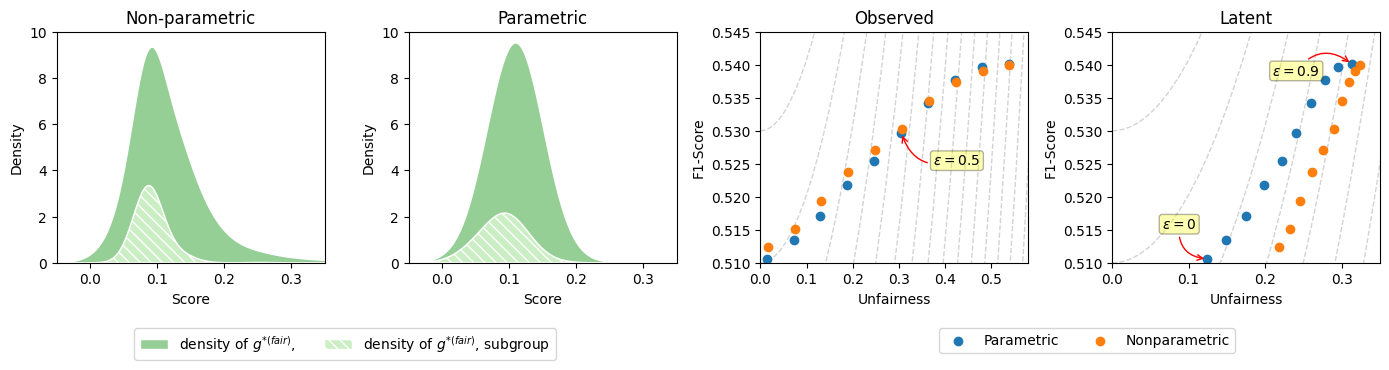}
\caption{Left set, scores for Violent Recidivism on \texttt{COMPAS} data set  corrected for subset of age-variable. Right set, $\varepsilon$-RI fairness and F1-score for observed and latent sensitive variable for the coverage task in the \texttt{folktables} data, in dashed gray, the fairness-risk trade-off lines.}\label{fig:application_fig}
\end{figure*}

For our numerical experiments we consider real data derived form the US-Census, gathered in the \texttt{folktables} package \cite{ding2021retiring} and the widely used \texttt{COMPAS} data set, collected by \cite{larson_angwin_kirchner_mattu_2016}. All source code, links to data, simulation details and specifications for the machines used throughout the experiments can be found on the code repository\footnote{\href{https://github.com/ParamFair/submission_974}{github.com/ParamFair/submission\_974}}.

We highlight two core properties, where domain knowledge can be incorporated in the estimation. The first set of experiments shows how a parametric form can help lessen the unfairness for a latent sensitive variable. The second experiment illustrates how prior knowledge can be used when the training data contains errors and cannot be efficiently corrected due to few available data points. For the simulations, we use a LightGBM \cite{ke2017lightgbm} base model and average our results across ten Monte-Carlo simulations.

\input{AISTATS2024PaperPack/tabs/tab_sims}
\subsection{Presence Of Latent Sensitive Variables}

A common, yet understudied problem in fairness applications is the absence of observable sensitive subgroups. This can arise either because the sensitive variable is not recorded due to regulatory concerns or when a variable is only available in an aggregated form. As shown in the introduction, a distributional constraint can help mitigate this issue. To illustrate the use on the \texttt{COMPAS} dataset, we estimate the scores for violent-recidivism and correct for the categorical age variable, however for the training phase we only observe whether an individual is middle aged or not (\textit{Observed} sensitive variable). During the test phase, we also evaluate the unfairness with respect to the \textit{Latent} sensitive variable, which is defined as the indicator that someone is member of the higher aged category. We evaluate the predictive performance and unfairness for the uncorrected method, the standard (nonparametric) approach and the parametric estimator proposed here. As a base distribution we opt for a Gaussian. We repeat the experiment for the \textit{ACSPublicCoverage} classification task from the \texttt{folktables} package for sunbelt states, but use a Beta as parametric form in this case. Here, the observed sensitive variable is a dummy indicating whether someone earns below 45,000\$ and the latent sensitive variable is an additional indicator whether someone earns less than 15,000\$. Results are summarized in Table \ref{tab:sims} with means and standard deviations reported, and illustrated in Figure \ref{fig:application_fig}. Whereas there is a slight decrease in predictive accuracy for the parametric version, it effectively helps mitigating the bias induced in the latent variable when compared to the standard nonparametric approach.

\subsection{Prior Knowledge Of Measurement Error}

A further application where domain knowledge might be useful is when data is either unavailable in large quantities or if the training data contains errors. We conduct a simple simulation on the \texttt{folktables} dataset predicting log wages (the \textit{ACSIncome} variable in its continuous form). We first estimate a fair parametric model based on the Gumbel distribution on data from the state of California. We then suppose our goal is to estimate the wages of the state of Texas, but that the training data contains measurement errors drawn from a Gamma$(s=1,s=0.5)$ distribution on various percentages (0,25\%,50\%,75\%) of the training data. We attempt to correct this using the estimated fair Gumbel parameters. This has the advantage that it is not dependent on the input variables, as other approaches such as transfer learning would be. The performance metrics, based on the mean squared error (MSE), are reported in Table \ref{tab:sims}. If the data is not corrupted, the procedure unsurprisingly adds to the prediction error. However, it significantly decreases estimation errors in the presence of error in the training data, presenting an attractive use-case for incorporating domain knowledge.

\section{CONCLUSION}
\label{sec:conclusion}
Applications of algorithmic fairness mostly consider a single and straightforward fairness measure. However, correcting for one source of bias might inadvertently propagate other biases in the supposedly fair predictions. Further, the agnostic approach of most procedures limits the incorporation of domain knowledge for the resulting predictive distribution. In this article, we show how imposing a parametric constraint can help alleviate both this issues. To the best of our knowledge, we are the first to consider such shape restrictions in algorithmic fairness. Our theoretical results show that these parametric estimators converge to the optimal values and at the same time we were able to bound the total budget necessary as compared to the optimal case. Whereas our results are interesting in their own rights, they also open up the possibility for future research. As different shape restrictions result in different intermediate solutions, a thorough analysis of the effects of different distributions is necessary to further our understanding of such restrictions.

\bibliographystyle{plainnat}
\bibliography{biblio}

\clearpage
\include{AISTATS2024PaperPack/proofs_arxiv}

\end{document}

%% file: AISTATS2024PaperPack/tabs/tab_sims.tex
\begin{table}[htbp]
\footnotesize
\begin{center}
\def\arraystretch{1.55}%
\setlength\tabcolsep{1.8pt}
\resizebox{\columnwidth}{!}{%
\begin{tabular}{|c c c c|}
\hline
  & \textbf{\textit{Uncorrected}} & \textbf{\textit{Standard}} & \textbf{\textit{Parametric}} \\
\hline
 & \multicolumn{3}{c|}{\textit{Classification - \texttt{Compas} - Normal }}\\
\cline{2-4}
Observed & 0.032 $\pm$ 0.023 & 0.026 $\pm$ 0.014 & 0.012 $\pm$ 0.007 \\
Latent & 0.111 $\pm$ 0.084 & 0.116 $\pm$ 0.086 & 0.045 $\pm$ 0.025\\
F1 & 0.221 $\pm$ 0.078 & 0.231 $\pm$ 0.080 & 0.227 $\pm$ 0.078\\
& \multicolumn{3}{c|}{\textit{Classification - \texttt{folktables} - Beta}}\\
\cline{2-4}
Observed & 0.586 $\pm$ 0.005 & 0.016 $\pm$ 0.007 & 0.038 $\pm$ 0.022 \\
Latent & 0.328 $\pm$ 0.004 & 0.213 $\pm$ 0.004 & 0.120 $\pm$ 0.002\\

F1, $\varepsilon$=0.00 & 0.538 $\pm$ 0.001 & 0.516 $\pm$ 0.004 & 0.513 $\pm$ 0.003\\
F1, $\varepsilon$=0.25 & 0.538 $\pm$  0.001& 0.519 $\pm$ 0.003 & 0.517 $\pm$ 0.003\\
F1, $\varepsilon$=0.50 & 0.538 $\pm$ 0.001& 0.528 $\pm$ 0.003 & 0.527 $\pm$ 0.003\\
F1, $\varepsilon$=0.75 & 0.538 $\pm$ 0.001& 0.535 $\pm$ 0.002 & 0.536 $\pm$ 0.002\\

& \multicolumn{3}{c|}{\textit{Regression - \texttt{folktables} - Measurement Error}}\\
\cline{2-4}
MSE - 0\% & N/A & 0.553 $\pm$ 0.006 & 0.569 $\pm$ 0.007 \\
MSE - 25\% & N/A & 0.711 $\pm$ 0.007 & 0.709 $\pm$ 0.007 \\
MSE - 50\% & N/A & 1.873 $\pm$ 0.014 & 1.737 $\pm$ 0.016 \\
MSE - 75\% & N/A & 5.704 $\pm$ 0.027 & 4.764 $\pm$ 0.027 \\
\hline
\end{tabular}
}
\end{center}
\label{tab:sims}
\caption{Results for simulations, indicated are means over the simulations with standard errors reported. Note that the \textit{Uncorrected} column for the folktables classification task stays constant across all $\varepsilon$ values as it is the basis of interpolation. }
\end{table}

%% file: AISTATS2024PaperPack/proofs_arxiv.tex
\onecolumn

\begin{center}
\bf{\Large Parametric Fairness with Statistical Guarantees: \\
Supplementary Materials}
\end{center}

\appendix

\section{Broader Impact}
 Our work is centered around fairness, which is a goal we sincerely believe all model should strive to achieve. Nevertheless, to ensure fairness in models, one needs to define unfairness as its counterpart. This naturally leads to a conundrum when performing research on the topic. On one hand, we would like our models to be fair, but to analyse the differences and show an improvement, we first need to create an unfair outcome. As has been shown in the past, simply ignoring the sensitive attributes does not solve the problem of bias in the data. Further, as more flexible methods make their way into practical applications, this issue is only bound to increase. Hence it is our conviction that estimating intentionally unfair models (by for example including sensitive variables explicitly in the training phase) is ethically justifiable if the goal is to provide a truly fair estimation. In that sense our work contributes to achieving fairness, and does not create new risks by itself. 

In our empirical application, we consider both data for wages and criminal justice. As has been discussed numerous times before in the media and academic research, there is a discrepancy between scores and predictions obtained for sensitive groups in these data sets. To avoid any misguided interpretation, we refrain from specifying the features directly, and merely use a subset of sensitive features to demonstrate the effectiveness of our method. Our goal here is to contribute to the body of research aiming to correct biases from arbitrary machine learning models and hence we put our focus on the metrics associated with this. In theory, we could use arbitrary data, but in the spirit of easier comparability, we opted to use well-known and publicly accessible data sets. We believe the safeguards taken and the use of the data in this context justify its use. 

\section{Proofs}

In this section, we gather the proofs of our results. Section~\ref{subappx:tech} covers essential fairness results in non-parametric fair regression. The proof for parametric fair predictors in the exact fairness framework is detailed in Section~\ref{subappx:main}. We also extend these results to approximate fairness in Section~\ref{subappx:approx} and discuss the theoretical properties of our estimation procedure in Section~\ref{subappx:cvg}.

\subsection{Fair Regression}\label{subappx:tech}

Recall that $g^*(\boldsymbol{X}, S) = \mathbb{E}[Y | \boldsymbol{X}, S]$ where we consider the general regression problems (can be extended to vector-valued problems)
\begin{equation*}\label{eq:RegGeneral}
   Y = g^{*}(\boldsymbol{X}, S) + \zeta \enspace,
\end{equation*}
with $\zeta\in\mathbb{R}$ a zero mean noise. The following lemma is adapted from \cite{gouic2020projection}. 

\begin{lem}[\cite{gouic2020projection}]\label{lem:ERisk}
    For a subclass of regressors $\mathcal{H}\subset\mathcal{G}$, if for any $h\in\mathcal{H}$ and $g\in\mathcal{G}$, where,
    $$
    \nu_{g|s} = \nu_{h|s} \quad \text{for all } s\in\mathcal{S}\enspace,
    $$
    we have $g\in\mathcal{H}$, then we can derive the associated excess-risk as,
    \begin{equation}
        \mathcal{E}(\mathcal{H}) = \inf_{h \in \mathcal{H}}\sum_{s\in\mathcal{S}}p_s \mathcal{W}_2^2\left(\nu_{g^*|s}, \nu_{h|s}\right)\enspace.
    \end{equation}
\end{lem}

Considering Lemma~\ref{lem:ERisk}, as $\mathcal{G}^0$ and $\mathcal{G}^0_\Theta$ represent the subclass of DP-fair regressors, we can easily deduce Eq.~\eqref{eq:FairExcessRisk} and Eq.~\eqref{eq:ExcessRiskParam}:
    \begin{equation*}
        \mathcal{E}(\mathcal{G}^0) = \inf_{g \in \mathcal{G}}\sum_{s\in\mathcal{S}}p_s \mathcal{W}_2^2\left(\nu_{g^*|s}, \nu_{g}\right)
    \quad \text{ and }\quad 
        \mathcal{E}(\mathcal{G}_\Theta^0) = \inf_{g_\theta \in \mathcal{G}_\Theta}\sum_{s\in\mathcal{S}}p_s \mathcal{W}_2^2\left(\nu_{g^*|s}, \nu_{g_\theta}\right)\enspace.
    \end{equation*}

\subsection{Parametric Fair Predictor}\label{subappx:main}

\begin{proof}[\textbf{Proof of Proposition 3.1.}]

Following Eq.~\eqref{eq:ExcessRiskParam}, let $g_\theta \in \mathcal{G}_\Theta$ and applying the Minkowski inequality, we have
        \begin{align*}
            \mathcal{W}_2^2(\nu_{g^*|s}, \nu_{g_\theta}) \leq 2 \mathcal{W}_2^2(\nu_{g^*|s}, \nu_{g^{(0)*}}) + 2 \mathcal{W}_2^2(\nu_{g^{(0)*}}, \nu_{g_\theta}) \quad \text{for all }s\in\mathcal{S} \enspace.
        \end{align*}
        Hence, if we denote $g_\theta^*$ a minimizer of the r.h.s. of the above equation, then
        \begin{align*}
            \inf_{g_\theta \in \mathcal{G}_\Theta} \sum_{s\in\mathcal{S}} p_s \mathcal{W}_2^2(\nu_{g^*|s}, \nu_{g_\theta}) \leq 
            \sum_{s\in\mathcal{S}} p_s \mathcal{W}_2^2(\nu_{g^*|s}, \nu_{g_\theta^*}) \leq 2\sum_{s\in\mathcal{S}} p_s \mathcal{W}_2^2(\nu_{g^*|s}, \nu_{g^{(0)*}}) + 2 \mathcal{W}_2^2(\nu_{g^{(0)*}}, \nu_{g_\theta^*})\enspace,
        \end{align*}
        which implies:
        \begin{equation*}
            \mathcal{E}(\mathcal{G}^0) \leq \mathcal{E}(\mathcal{G}_\Theta^0) \leq 2\left(\mathcal{E}(\mathcal{G}^0) + \inf_{g_\theta \in \mathcal{G}_\Theta} \mathcal{W}_2^2(\nu_{g^{(0)*}}, \nu_{g_\theta})\right)\enspace.
        \end{equation*}

Finally, since $\mathbb{E}[g^*] = \mathbb{E}[g^{(0)*}]$, we derive the bounds on the following budget deviation $\mathcal{D}(g_\theta) = \mathbb{E}[g_\theta - g^{(0)*}]$ instead and we can infer the following upper bound
        \begin{align*}
            \mathcal{D}(g_\theta)^2 \leq \mathbb{E}[(g_\theta - g^{(0)*})^2] \leq \mathcal{W}_2^2 (\nu_{g^{(0)*}}, \nu_{g_\theta}), \quad \text{for all } g_\theta \in \mathcal{G}_\Theta\enspace,
        \end{align*}
        where we used Jensen's inequality on the left side inequality. In words this means that 
        $\mathcal{D}(g_\theta)^2$ is controlled by the distributional distance between $g_\theta$ and $g^{(0)*}$. Therefore, within $\mathcal{G}_\Theta$, 
 $g_\theta^{*}$ corresponds to the best solution for budget stability in terms of the $L_2$ norm. Additionally, replacing $g_\theta$ with $g_\theta^{(0)*}$ in the inequality above concludes the proof.
\end{proof}

\subsection{Extension to Approximate Fairness}
\label{subappx:approx}

Before going into the main proposition, let us present and prove the following lemma:
\begin{lem}[Risk-unfairness trade-off for interpolation cases]
    For any interpolated predictor of the form:
    $$
    g^{(\varepsilon)}(\boldsymbol{X}, S) = (1-\varepsilon)\cdot g(\boldsymbol{X}, S) + \varepsilon \cdot g^{*}(\boldsymbol{X}, S)
    $$
    with $g\in\mathcal{G}$, we have
    $$
    \mathcal{R}(g^{(\varepsilon)}) = (1-\varepsilon)^2 \cdot \mathcal{R}(g)\quad \text{ and }\quad \mathcal{U}(g^{\varepsilon}) = \varepsilon \cdot \mathcal{U}(g^*)\enspace.
    $$
\end{lem}
\begin{proof}
    For any $g \in \mathcal{G}$, the associated $g^{(\varepsilon)}$ verifies:
        \begin{align*}
            \mathcal{R}(g^{(\varepsilon)}) &= \mathbb{E}[((1-\varepsilon)\cdot g(\boldsymbol{X}, S) + \varepsilon\cdot g^* (\boldsymbol{X}, S) - g^*(\boldsymbol{X}, S))^2] \\
            &=
            \mathbb{E}[((1-\varepsilon)\cdot(g(\boldsymbol{X}, S) - g^*(\boldsymbol{X}, S)))^2] \\
            &= (1-\varepsilon)^2\cdot \mathcal{R}(g) \\
    \mathcal{U}(g^{\varepsilon}) &= \max_{s\in \mathcal{S}}\mathcal{W}_1(\nu_{g^{(\varepsilon)}{|s}}, \nu_{g^{(\varepsilon)}}) = \max_{s\in \mathcal{S}} \varepsilon\mathcal{W}_1(\nu_{g^*|s}, \nu_{g^*}) + 0 = \varepsilon \times \mathcal{U}(g^*)
        \end{align*}
\end{proof}

Thus, through this Lemma we have demonstrated Lemma~\ref{lem:RiskUnf} and Proposition~\ref{prop:ParamFairPredictorApprox}-i).

\begin{proof}[\textbf{Proof of Proposition~\ref{prop:ParamFairPredictorApprox}.}]
We divide the proof into parts according to the three sub-points i), ii) and iii):
    \begin{enumerate}
        \item[i)] {proven above.
        }
        \item[ii)]{
        To derive the approximate parametric fairness, we first define $g^{(\varepsilon)*}$ and $\Tilde{h}^{(\varepsilon)}$ as the solutions to the two following optimization problems:
        \begin{align*}
            \mathcal{E}(\mathcal{H}^\varepsilon) & = \sum_{s\in\mathcal{S}} p_s \mathcal{W}_2^2(\nu_{g^*|s}, \nu_{\Tilde{h}^{(\varepsilon)}| s}) \\
            \mathcal{E}(\mathcal{G}^\varepsilon) & = \sum_{s\in\mathcal{S}} p_s \mathcal{W}_2^2(\nu_{g^*|s}, \nu_{g^{(\varepsilon)*}|s}) = (1-\varepsilon)^2 \sum_{s\in\mathcal{S}} p_s \mathcal{W}_2^2(\nu_{g^*|s}, \nu_{g^{(0)*}})
        \end{align*}
        Then, by the definition of $\mathcal{G}^\varepsilon$ and $\mathcal{H}^\varepsilon$, there exists $g^{(0)*}$ and $g^{(0)*}_\theta$ such that:
        \begin{align*}
            g^{(\varepsilon)*}(\boldsymbol{X}, S) & = (1-\varepsilon)\cdot g^{(0)*}(\boldsymbol{X}, S) + \varepsilon\cdot g^*(\boldsymbol{X}, S) \enspace,\\
            \Tilde{h}^{(\varepsilon)}(\boldsymbol{X}, S) &= (1-\varepsilon)\cdot g_\theta^{(0)*}(\boldsymbol{X}, S) + \varepsilon\cdot g^*(\boldsymbol{X}, S) \enspace.
        \end{align*}
        Further $g^{(0)*}$ corresponds to the Wasserstein barycenter defined in proposition 2.5 and $g_\theta^{(0)*}$ corresponds to the $\min$ specified in i). Conversely, we also have:
        \begin{align*}
            \mathcal{W}_2^2(\nu_{g^*|s}, \nu_{h^{(\varepsilon)}|s}) \leq 2 \mathcal{W}_2^2(\nu_{g^*|s}, \nu_{g^{(\varepsilon)*}|s}) + 2\mathcal{W}_2^2(\nu_{g^{(\varepsilon)*}|s}, \nu_{h^{(\varepsilon)}|s})\enspace.
        \end{align*}
        
        Together with the fact that $\mathcal{E}(\mathcal{G}^{\varepsilon}) = (1-\varepsilon)^2\mathcal{E}(\mathcal{G}^0)$ we recover the main statement:
        
        \begin{align*}
            0 \leq \mathcal{E}(\mathcal{G}^\varepsilon) \leq \mathcal{E}(\mathcal{H}^\varepsilon) \leq 2(1-\varepsilon)^2\big(\mathcal{E}(\mathcal{G}^0) + \inf_{g_\theta \in \mathcal{G}_\Theta} \mathcal{W}_2^2(\nu_{g^{(0)*}}, \nu_{g_\theta}) \big)\enspace.
        \end{align*}
        }
        \item[iii)] 
        Since $\mathbb{E}[g^*] = \mathbb{E}[g^{(0)*}]$, it becomes apparent that for any $\varepsilon\in[0, 1]$, the equality $\mathbb{E}[g^*] =\mathbb{E}[g^{(\varepsilon)*}]$ holds true. Further, using Jensen's inequality, we derive straightforwardly the upper-bound on the squared budget deviation:
        \begin{align*}
            \mathcal{D}(h^{(\varepsilon)} )^2 = \mathbb{E}(h^{(\varepsilon)} - g^{(\varepsilon)*})^2 \leq \mathbb{E}[(h^{(\varepsilon)}  - g^{(\varepsilon)*})^2] \leq \mathcal{W}_2^2 (\nu_{g^{(\varepsilon)*}}, \nu_{h^{(\varepsilon)} }), \quad \text{for all } h^{(\varepsilon)}  \in \mathcal{H}^{(\varepsilon)}\enspace.
        \end{align*}
        Finally, replacing $h^{(\varepsilon)}$ with $\Tilde{h}^{(\varepsilon)}$ in the inequality above concludes the proof.
    \end{enumerate}
\end{proof}

\subsection{Statistical Guarantees}\label{subappx:cvg}

First, we establish the estimation guarantee before delving into the fairness guarantee. Note that we have adapted the estimation guarantee sequentially from \cite{gouic2020projection}, and \cite{bernton2019parameter} to account for the parametric framework. 

First, recall that throughout the paper, Assumption~\ref{assu:general} must hold.

\paragraph{Estimation Guarantee for $\varepsilon = 0$}

We denote $\Hat{\nu}_g$ as the empirical measure of $\nu_g$. To outline our statistical guarantees, let us present all the necessary assumptions:

\begin{assu}[Smoothness \& Bound assumption]\label{assu:Lips}
    We assume that $(\Hat{g}(\cdot, s) )_{s\in\mathcal{S}}$ are uniformly Lipschitz and the estimator $\Hat{g}$ is bounded.
\end{assu}
With this assumption, we are then able to derive the following estimation guarantee for $\Hat{g}^{(0)}$, estimator of $g^{(0)*}$.
\begin{lem}[adapted from \cite{gouic2020projection} Thm.~8]
\label{lem:gouic}
    Under Assumption~\ref{assu:Lips}, and, as $n\to \infty$, assuming the following $L_2$-consistency
$$
\mathbb{E}[(g^*(\boldsymbol{X}, S) - \hat{g}(\boldsymbol{X}, S))^2] \to 0\enspace,
$$
we have both,
$$
\mathcal{W}_2^2 (\Hat{\nu}_{\hat{g}|s}, \nu_{g^{*}|s}) \to 0 \quad \text{and}\quad \mathcal{W}_2^2 (\Hat{\nu}_{\hat{g}^{(0)}|s}, \nu_{g^{(0)*}}) \to 0\quad\text{a.s.}
$$  
\end{lem}

Let us now consider the estimation guarantees in the parametric framework for $\Hat{g}_\theta$, estimator of $\Tilde{g}_\theta$. The Minimum Wasserstein Estimator (MWE) can be computationally intractable, especially in settings where simulating data from the model is possible but evaluating its density is not. In such cases, the Minimum Expected Wasserstein Estimator (MEWE) may offer a more computationally convenient alternative, where the expectation is replaced with a Monte Carlo approximation. Further, given that we explicitly want to sample from a parametric distribution, this approach is particularly attractive (see~\cite{bernton2019parameter} for more details). For simplicity, since it has been shown that MEWE converges to MWE and since in practice we have chosen a sufficiently large Monte Carlo sample to generate $\Hat{g}_\theta$, we omit the assumptions associated with Monte Carlo approximation.

\begin{assu}
\label{assu:mewe}
The assumptions are as follows:
    \begin{enumerate}
        \item $\mathcal{G}_\Theta$ is a family of continuous distributions.
        \item The map $\theta \mapsto \nu_{g_\theta}$ is continuous.
        \item For some $\alpha > 0$, the set $B_n(\alpha) = \{ \theta \in \Theta: \mathcal{W}_2(\Hat{\nu}_{\Hat{g}^{(0)}}, \nu_{g_\theta}) \leq \inf_{g_\theta\in\mathcal{G}_\Theta}\mathcal{W}_2(\Hat{\nu}_{\Hat{g}^{(0)}}, \nu_{g_\theta}) + \alpha \}$ is bounded.
    \end{enumerate}
\end{assu}

\begin{lem}[Existence and consistency, adapted from \cite{bernton2019parameter} Thm.~2.1 and Thm.~2.4]
\label{lem:bernton}
    Under Assumption~\ref{assu:Lips} and Assumption~\ref{assu:mewe}, as $n\to\infty$,
    $$
    \inf_{\Hat{g}_\theta\in\mathcal{G}_\Theta} \mathcal{W}_2(\Hat{\nu}_{\Hat{g}^{(0)}}, \Hat{\nu}_{\Hat{g}_{\theta}}) \to \inf_{g_\theta\in\mathcal{G}_\Theta}\mathcal{W}_2(\nu_{g^{(0)*}}, \nu_{g_\theta})\quad a.s.\enspace,
    $$
    and \textit{a.s.} for a large enough sample, the sets $\argmin{}_{\Hat{g}_\theta\in\mathcal{G}_\Theta} \mathcal{W}_2(\Hat{\nu}_{\Hat{g}^{(0)}}, \Hat{\nu}_{\Hat{g}_{\theta}})$ are non-empty and form a bounded sequence with
    $$
    \lim \sup_{n\to\infty} \argmin{\Hat{g}_\theta\in\mathcal{G}_\Theta} \mathcal{W}_2(\Hat{\nu}_{\Hat{g}^{(0)}}, \Hat{\nu}_{\Hat{g}_{\theta}}) \subset \argmin{g_\theta\in\mathcal{G}_\Theta}\mathcal{W}_2(\nu_{g^{(0)*}}, \nu_{g_\theta})\enspace.
    $$
\end{lem}


Note that the detailed rate of convergence, under stronger assumptions, can also be found in \cite{bernton2019parameter}.


\paragraph{Estimation Guarantee for $\varepsilon \geq 0$ (Corollary~\ref{coro:estimation})}

Since, with the geodesic interpolation, we have
$$
\Tilde{h}^{(\varepsilon)}(\boldsymbol{X}, S) = (1-\varepsilon)\cdot \Tilde{g}_\theta(\boldsymbol{X}, S) + \varepsilon \cdot g^{*}(\boldsymbol{X}, S)\enspace,
$$
and 
$$
\hat{h}^{(\varepsilon)}(\boldsymbol{X}, S) = (1-\varepsilon)\cdot \Hat{g}_\theta(\boldsymbol{X}, S) + \varepsilon \cdot \Hat{g}(\boldsymbol{X}, S)\enspace,
$$
by the above results (Lemma~\ref{lem:gouic} and Lemma~\ref{lem:bernton}), as $n\to+\infty$, we have directly
$$
\mathcal{W}_2\left(\Hat{\nu}_{\hat{h}^{(\varepsilon)}}, \nu_{\Tilde{h}^{(\varepsilon)}}\right) \to 0\quad a.s\enspace.
$$

\paragraph{Fairness Guarantee (Corollary~\ref{coro:fairness})} Recall Assumptions \ref{assu:Lips}-\ref{assu:mewe} and the above $L_2$-consistency of $\Hat{g}$.
Given $\varepsilon \in[0, 1]$, we observe asymptotic fairness with $\varepsilon$-RI. Indeed, since by construction $\Tilde{g}_\theta$ is independent to $S$ and by applying triangle inequality, we have for all $s\in\mathcal{S}$,
\begin{align*}
\mathcal{W}_1\left(\hat{\nu}_{\hat{h}^{(\varepsilon)}|s}, \hat{\nu}_{\hat{h}^{(\varepsilon)}}\right) &\leq \mathcal{W}_1\left(\hat{\nu}_{\hat{h}^{(\varepsilon)}|s}, \nu_{\Tilde{h}^{(\varepsilon)}|s}\right) + 
\mathcal{W}_1\left(\nu_{\Tilde{h}^{(\varepsilon)}|s}, \nu_{\Tilde{h}^{(\varepsilon)}}\right)+
\mathcal{W}_1\left(\nu_{\Tilde{h}^{(\varepsilon)}}, \hat{\nu}_{\hat{h}^{(\varepsilon)}}\right)\\
&\leq \varepsilon\times \mathcal{U}(g^*) + \mathcal{W}_1\left(\hat{\nu}_{\hat{h}^{(\varepsilon)}|s}, \nu_{\Tilde{h}^{(\varepsilon)}|s}\right) + \mathcal{W}_1\left(\nu_{\Tilde{h}^{(\varepsilon)}}, \hat{\nu}_{\hat{h}^{(\varepsilon)}}\right)\enspace,
\end{align*}
where the two right terms on the right side of the inequality tends towards zero. This convergence is achieved due to the fact that convergence in $\mathcal{W}_2$ implies convergence in $\mathcal{W}_1$.

